\documentclass[runningheads, envcountsame, a4paper]{llncs}
\usepackage[latin1]{inputenc}
\usepackage{amsmath}
\usepackage{amsfonts}
\usepackage{amssymb}
\usepackage[pdftex]{graphicx}
\usepackage{epstopdf}
\usepackage{booktabs}
\usepackage{multirow}
\usepackage{subcaption}
\usepackage{xcolor}
\usepackage[hyphens]{url}
\usepackage{comment}
\usepackage{microtype}

\usepackage[misc]{ifsym}

\usepackage{mathtools}

\DeclarePairedDelimiter{\diagfences}{(}{)}
\newcommand{\diag}{\operatorname{diag}\diagfences}
\newcommand{\mvec}{\operatorname{vec}\diagfences}
\newcommand{\tr}{\operatorname{tr}\diagfences}
\newcommand{\sign}{\operatorname{sign}\diagfences}

\usepackage{bbm}

\usepackage{mathtools}

\DeclarePairedDelimiter\floor{\lfloor}{\rfloor}

\newcommand{\dd}[1]{\mathrm{d}#1}

\usepackage{algpseudocode}
\usepackage[ruled,vlined]{algorithm2e}
\SetKwInOut{Parameter}{parameter}

\SetKwInput{KwData}{Input}
\SetKwInput{KwResult}{Outout}

\usepackage{hyperref}
\hypersetup{
  pdfinfo={
    Title={Towards Interpretable Multi-Task Learning Using Bilevel Programming},
    Author={Francesco Alesiani , Shujian Yu , Ammar Shaker, Wenzhe Yin},
    Subject={Interpretable MTL},
    Keywords={ML,MTL, Interpretable}
  }
}

\begin{document}

	
%
\title{Towards Interpretable Multi-Task Learning Using Bilevel Programming}

\toctitle{Towards Interpretable Multi-Task Learning Using Bilevel Programming}
\tocauthor{Francesco~Alesiani , Shujian~Yu , Ammar~Shaker, Wenzhe~Yin }

\titlerunning{Towards Interpretable Multi-Task Learning Using Bilevel Programming}
\authorrunning{F. Alesiani \and S. Yu \and A. Shaker \and W. Yin }

%
%

\author{Francesco Alesiani\inst{1}  (\Letter)  \and
	Shujian Yu\inst{1}  \and
	Ammar Shaker\inst{1}  \and
	Wenzhe Yin\inst{1}
}


%
%




\institute{NEC Laboratories Europe, $69115$ Heidelberg, Germany \\
	\url{http://www.neclab.eu} \thanks{ Manuscript accepted at 
ECML PKDD 2020.}\\
	\email{$\{$Francesco.Alesiani,Shujian.Yu,Ammar.Shaker$\}$@neclab.eu}
	\email{Wenzhe.Yin@stud.uni-heidelberg.de\thanks {Work done while at the NEC Laboratories Europe.}} 
}





%
\maketitle              
\setcounter{footnote}{0}

\begin{abstract}
Interpretable Multi-Task Learning can be expressed as learning a sparse graph of the task relationship based on the prediction performance of the learned models. 
Since many natural phenomenon exhibit sparse structures, enforcing sparsity on learned models reveals the underlying task relationship. Moreover, different sparsification degrees from a fully connected graph uncover various types of structures, like cliques, trees, lines, clusters or fully disconnected graphs.
In this paper, we propose a bilevel formulation of multi-task learning that induces sparse graphs, thus, revealing the underlying task relationships, and an efficient method for its computation.
We show empirically how the induced sparse graph improves the interpretability of the learned models and their relationship on synthetic and real data, without sacrificing generalization performance. Code at \url{https://bit.ly/GraphGuidedMTL}

	
\keywords{Interpretable machine learning  \and Multi-Task learning \and Structure learning \and Sparse graph \and Transfer Learning
}
\end{abstract}
\section{Introduction}
Multi-task learning (MTL) is an area of machine learning that aims at exploiting relationships among tasks to improve the collective generalization performance of all tasks. In MTL, learning of different tasks is performed jointly, thus, it transfers knowledge from information-rich tasks via task relationship \cite{zhang2014regularization} so that the overall generalization error can be reduced. MTL has been successfully applied in various domains ranging from transportation \cite{deng2017situation} to biomedicine \cite{li2018multi}.  The improvement with respect to learning each task independently is significant when each task has only a limited amount of training data \cite{argyriou2008convex}. 

Various multi-task learning algorithms have been proposed in the literature (see Zhang and Yang \cite{Zhang2017Survey} for a comprehensive survey on state-of-the-art 
methods). Feature learning approaches \cite{argyriou2007multi} and low-rank approaches 
\cite{chen2011integrating} 
assume all the tasks are related, which may not be true in real-world applications. Task clustering approaches 
\cite{kumar2012learning} 
can deal with the situation where different tasks form clusters. 
He et al. \cite{he_efficient_2019} propose a MTL method that is both accurate and efficient; thus applicable in presence of large number of tasks, as in the retail sector.
However, despite being accurate and scalable, these methods lack interpretability, when it comes to task relationship. 
Trustworthy Artificial Intelligence, is an EU initiative to capture the main requirements of ethical AI. Transparency and Human oversight are among the seven key requirements developed by the AI Expert Group 
\cite{intelligence_policy_2019}. 
Even if MTL improves performance w.r.t to individual models, predictions made by black-box methods can not be used as basis for decisions, unless justified by interpretable models \cite{guidotti_survey_2018,lipton_mythos_2018}.

Interpretability can be defined locally. LIME \cite{ribeiro_why_2016} and its generalizations \cite{lundberg_unified_2017} are local method that extract features for each test sample that most contribute to the prediction and aim at finding a sparse model that describes the decision boundary. These methods are applied downstream of an independent black-box machine learning method that produces the prediction. Inpretability can also achieved globaly. For example linear regression, logistic regression and decision tree \cite{guidotti_survey_2018,lipton_mythos_2018} are considered interpretable\footnote{inteterpretability depends also on the application, where for example it may be associated with weights being integer, and can be defined thus differently}, since their parameters are directly interpretable as weights on the input features. Global interpretability, on the other hand, could reduce accuracy. In general, MTL methods \cite{ruder2017overview,liu_multi-task_2018} are not directly interpretable, unless single tasks are learned as linear models which are considered interpretable due to their simplicity \cite{saha_online_2011,goncalves_multi-task_2014,murugesan_adaptive_2016}. This property, however, is no longer guaranteed when tasks and their relations are learned simultaneously, mainly because the relative importance of task relationship is not revealed. 
Since natural phenomena are often characterized by sparse structures, we explore the interpretability resulting from imposing the relationship among tasks to be sparse.





To fill the gap of interpretabilty in MTL, this paper introduces a novel algorithm, named \textbf{G}raph \textbf{G}uided \textbf{M}ulti-\textbf{T}ask regression \textbf{L}earning (GGMTL). It integrates the objective of joint interpretable (i.e. sparse) structure learning with the multi-task model learning.
GGMTL enjoys a closed-form hyper-gradient computation on the edge cost; it also provides a way to learn the graph's structure by exploiting the linear nature of the regression tasks, without excessively scarifying the accuracy of the learned models. The detailed contribution of this paper is multi-fold:

\begin{description}
	\item \textbf{Bilevel MTL Model:} A new model for the joint learning of sparse graph structures and multi-task regression that employs graph smoothing on the prediction models (Sec.\ref{sec:bilevel});
	\item \textbf{Closed-form hyper-gradient:} Presents a closed-form solution for the hyper-gradient of the graph smoothing multi-task problem (Sec.\ref{sec:hg}) 
    \item \textbf{Interpretable Graph:} The learning of interpretable graph structures;
	\item \textbf{Accurate Prediction:} Accurate predictions on both synthetic and real-world datasets despite the improved interpretability (Sec.\ref{sec:experiment}); 
	\item \textbf{Efficient computation:} of the hyper-gradient of the proposed bilevel problem; 
	\item \textbf{Efficient method:} that solves the proposed bilevel problem (Sec.\ref{sec:l22-GGMTL},\ref{sec:l2-GGMTL});
	\item \textbf{Veracity measures:} { to evaluate the fidelity of learned MTL graph structure (Sec.\ref{sec:veradicity})}

\end{description}


\section{Related work}

\subsection{Multi-task structure learning}
Substantial efforts have been made on estimating model parameters of each task and the mutual relationship (or dependency) between tasks. Usually, such relationship is characterized by a dense task covariance matrix or a task precision matrix (a.k.a., the inverse of covariance matrix). Early methods 
(e.g.,~\cite{obozinski2010joint}) 
assume that all tasks are related to each other. However, this assumption is over-optimistic and may be inappropriate for certain applications, where different tasks may exhibit different degrees of relatedness. To tackle this problem, more elaborated approaches, such as clustering of tasks (e.g.,~\cite{jacob2009clustered}) or hierarchical structured tasks (e.g.,~\cite{han2015learning}) have been proposed in recent years.

The joint convex learning of multiple tasks and a task covariance matrix was initialized in Multi-Task Relationship Learning (MTRL)~\cite{zhang2010convex}. Later, the Bayesian Multi-task with Structure Learning (BMSL)~\cite{goncalves2019bayesian} improves MTRL by introducing sparsity constraints on the inverse of task covariance matrix under a Bayesian optimization framework. On the other hand, the recently proposed multi-task sparse structure learning (MSSL)~\cite{gonccalves2016multi} directly optimizes the precision matrix using a regularized Gaussian graphical model. One should note that, although the learned matrix carries partial dependency between pairwise tasks, there is no guarantee that the learned task covariance or prediction matrix can be transformed into a valid graph Laplacian~\cite{dong2016learning}. From this perspective, the learned task structures from these works suffer from poor interpretability.

\subsection{Bilevel optimization in machine learning}

Bilevel problems~\cite{colson2007overview} raise when a problem (outer problem) contains another optimization problem (inner problem) as constraint. Intuitively, the outer problem (master) defines its solution by predicting the behaviour of the inner problem (follower).
In machine learning, hyper-parameter optimization tries to find the predictive model's parameters $w$, with respect to the hyper-parameters vector $\lambda$ that minimizes the validation error. This can be mathematically formulated as the bilevel problem
\begin{subequations} \label{eq:biproblem_general}
	\begin{eqnarray}
	\min _{\lambda} F(\lambda) &=& E_{s \sim D^{\text{val}}} \{ f(w_{\lambda},\lambda,s)\} \label{biproblem_general_outter} \\
	\text{s.t.} ~ w_{\lambda} &=& \arg \min_{w} E_{s' \sim D^{\text{tr}}} \{ g(w,\lambda,s') \} \label{biproblem_general_inner} ,
	\end{eqnarray}
\end{subequations}
The outer objective is the minimization of the generalization error $ E_{s \sim D^{\text{val}}} \{ f(w_{\lambda}, \allowbreak \lambda, s) \}$ on the hyper-parameters and validation data $D^{\text{val}}$, whereas $E_{s \sim D^{\text{tr}}} \{ g(w, \allowbreak \lambda, s)\}$ is the regularized empirical error on the training data $D^{\text{tr}}$, see \cite{franceschi_bilevel_2018}, where $D^{\text{val}} \bigcup D^{\text{tr}} = D$. 
The bilevel optimization formulation has the advantage of allowing to optimize two different cost functions (in the inner and outer problems) on different data (training/validation), thus, alleviating the problem of over-fitting and implementing an implicit cross validation procedure.

In the context of machine learning, bilevel optimization has been adopted mainly as a surrogate to the time-consuming cross-validation which always requires grid search in high-dimensional space. For example, \cite{jenni2018deep} formulates cross-validation as a bilevel optimization problem to train deep neural networks for improved generalization capability and reduced test errors. \cite{frecon2018bilevel} follows the same idea and applies bilevel optimization to group Lasso~\cite{yuan2006model} in order to determine the optimal group partition among a huge number of options.

Given the flexibility of bilevel optimization, it becomes a natural idea to cast multi-task learning into this framework. Indeed, \cite[Chapter~5]{kunapuli2008bilevel} first presents such a formulation by making each of the individual hyperplanes (of each task) less susceptible to variations within their respective training sets. However, no solid examples or discussions are provided further. This initial idea was significantly improved in~\cite{flamary2014learning}, in which the outer problem optimizes a proxy of the generalization error over all tasks with respect to a task similarity matrix and the inner problem estimates the parameters of each task assuming the task similarity matrix is known.

\section{Graph guided MTL}


\subsection{Bilevel multi-tasking linear regression with graph smoothing}\label{sec:bilevel}
We consider the problem of finding regression models $\{w_i\}$ for $n$ tasks, with input/output data $\{(X_i,y_i)\}_{i=1}^{n}$, where $X_i \in R^{N_i \times d}$, $y_i \in R^{N_i \times 1}$ and $d$ is the feature size, while $N_i$ is the number of samples for the $i$th task\footnote{In the following we assume for simplicity $N_i=N$ for all tasks, but results extend straightforward.}. We split the data into validation $\{(X^\text{val}_i,y^\text{val}_i)\}_{i=1}^{n}$ and training $\{(X^\text{tr}_i,y^\text{tr}_i)\}_{i=1}^{n}$ sets and formulate the problem as a bilevel program:
\begin{subequations} \label{eq:bi_main}
\begin{eqnarray} 
\min_{e} && \sum_{i \in [n]} || X_i^\text{val} w_{e,i} -y_i^\text{val} ||^2 + \xi ||e||^2_2 + \eta ||e||_1 + \gamma H(e) \label{eq:bi_main_out}\\
V_e &=& \arg \min _{V} \sum_{i \in [n]} || X_i^\text{tr} w_i -y_i^\text{tr} ||^2 + \frac1{2} \lambda \tr { V^T L_e V } \; , \label{eq:bi_main_in}
\end{eqnarray}
\end{subequations}
where $V= [w_1^T,\dots,w_n^T]^T$ is the models' vectors, $L_{e} = \sum_{ij \in G} e_{ij} (d_i-d_j)(d_i-d_j)^T =  E  \diag{e} E^T$ is the Laplacian matrix defined using the incident matrix $E$, $e=\mvec{[e_{ij}]}$ is the edge weight vector with $[e_{ij}]$ being the adjacent matrix, and $d_i$ is the discrete indicator vector which is zero everywhere except at the $i$-th entry. We use $[n]$ for the set $\{1,\dots,n\}$. The regularization term in the inner problem is the {\it Dirichlet energy} \cite{belkin_laplacian_2002}
\begin{equation} \label{eq:energy}
\tr { V^T L_e V } = \sum_{ij \in G} e_{ij}||w_i-w_j||_2^2 \; ,
\end{equation}
where $G$ is the graph whose Laplacian matrix is $L_e$. $H(e) = - \sum_{ij \in G} ( |e_{ij}| \ln |e_{ij}| - |e_{ij}|)$  is the un-normalized entropy of the edge values.

{ The inner problem (model learning) aims at finding the optimal model for a given structure (i.e. graph), while the outer problem (structure learning)} aims at minimizing a cost function that includes two terms: (1) the learned model's accuracy on the validation data, and (2) the sparseness of the graph. We capture the sparseness of the graph with three terms: (a) the $\ell_2^2$ norm of the edge values, measuring the energy of the graph, (b) the $\ell_1$ norm measuring the sparseness of the edges, and (c) $H(e)$ measuring the entropy of the edges. In the experiment, we limit the edges to have values in the interval $[0,1]$, which can be interpreted as a relaxation of the \textit{mixed integer non-linear programming} problem when $e_{ij} \in \{0,1\}$ as defined in Eq.\ref{eq:bi_main}. The advantage of formulating the MTL learning as a bilevel program (Eq.\ref{eq:bi_main}) is the ability to derive a closed-form solution for the hyper-gradient (see Thm.\ref{th:hp}). Moreover, for a proper choice of the regularization parameter ($\ell_1$), all edge weights have a closed-form solution (see Thm.\ref{th:sol}). For the general case, we propose a gradient descent algorithm (Alg.\ref{alg:GGMTL}). 
Entropy regularization term has superior sparsification performance to the $\ell_1$ norm regularization \cite{huang_sparse_2019}, thus, the latter can be ignored during hyper-parameter search to reduce the search space at the expense of a improved flexibility.



For simplicity, we define the functions:
\begin{subequations} \label{eq:bi_proxy}
\begin{eqnarray}
f(V_e,e) &=& \sum_{i \in [n]} || X_i^\text{val} w_{e,i} -y_i^\text{val} ||^2 + \xi ||e||^2_2 + \eta ||e||_1 + \gamma H(e)  \label{eq:bi_outer} \\
g(V,e) &=& \sum_{i \in [n]} || X_i^\text{tr} w_i -y_i^\text{tr} ||^2 + \frac1{2} \lambda \tr { V^T L_e V } \; ,\label{eq:bi_inner}
\end{eqnarray}
\end{subequations}
which allow us to write the bilevel problem in the compact form:

\begin{equation} \label{eq:bi2}
\min_{e}  f(V_e,e) ~ \text{s.t.} ~ V_e = \arg \min _{V} g(V,e) \;.
\end{equation}

{ The proposed formulation optimally selects the sparser graph among tasks that provides the best generalization performance on the validation dataset.}

\begin{algorithm}
	\SetNoFillComment
	\SetKwInOut{Input}{Input}
	\SetKwInOut{Output}{Output}
	\Input{$\{X_t, y_t\}$ for $t=\{1, 2, ..., n\}$,  $\xi, \eta,\lambda, \nu$}
	\Output{$V= [w_1^T, ..., w_n^T]^T, L_e$}
	\For{$i\gets1$ \KwTo $n$}{
		Solve $w_i$ by Linear Regression on $\{X_i, y_i\}$
	}
	Construct $k$-nearest neighbor graph $G$ on $V$\;
	Construct $E$ the incident matrix of $G$\;
	\tcp{validation-training split}
	$\{X^\text{tr}_t, y^\text{tr}_t\},\{X^\text{val}_t, y^\text{val}_t\} \gets \text{split}( \{X_t, y_t\})$  \; 		
	\While{not converge}{
		\tcp{compute hyper-gradient}
		compute $\dd _{e} f(V_{e^{(t)}},e^{(t)})$ using Eq. (\ref{eq:hg}), (where, with $\ell_2$ norm,  Eq. (\ref{eq:hg}) is computed using $e=e \circ l(V)$ of Eq.\ref{eq:elle} and alternating with solution of Eq.\ref{eq:bi_ccmtl_sub} (given by Eq.\ref{eq:sol_inner} or in Thrm.\ref{th:hp}).\;
		\tcp{edges' values update}		
		Update $e$: $e^{(t+1)} = [ e^{(t)}+\nu \dd _{e} f(V_{e^{(t)}},e^{(t)})]_{+}$\;
	}
	\tcp{Train on the full datasets with alternate optimization}
	Solve Eq.\ref{eq:bi_ccmtl} on $\{X_t, y_t\}$\;
	\Return $V,L_e$\;
	\caption{GGMTL: $\ell_2$-GGMTL, $\ell^2_2$-GGMTL}
	\label{alg:GGMTL}
\end{algorithm}

\subsection{The $\ell_2$ norm-square regularization $\ell_2^2$-GGMTL algorithm} \label{sec:l22-GGMTL}
We propose an iterative approach that computes the hyper-gradient of $f(V_e,e)$ (eq.\ref{eq:bi_outer}) with respect to the graph edges (the hyper-parameters); this hyper-gradient is then used for updating the hyper-parameters based on the gradient descend method, i.e.,
\begin{equation} \label{eq:hg_update}
e^{(t+1)} = e^{(t)}+\nu \dd _{e} f(V_{e^{(t)}},e^{(t)}) \; ,
\end{equation}
where $\dd _{e}$ is the hyper-gradient and $\nu$ is the learning rate. Algorithm Alg.\ref{alg:GGMTL} depicts the structure of the GGMTL learning method, where $[x]_{+}=\max(0,x)$. 
The stopping criterion is evaluated on the convergence of the validation and training errors. As a final step, the tasks's models are re-learned on all training and validation data based on the last discovered edge values. 

\subsection{The $\ell_2$ norm regularization $\ell_2$-GGMTL algorithm}\label{sec:l2-GGMTL}
The energy smoothing term Eq.\ref{eq:energy} in the inner problem of Eq.\ref{eq:bi_main} is a quadratic term.
However, if two models are unrelated, but connected by an erroneous edge, this term grows quadratically dominating the loss. To reduce this undesirable effect, a term proportional to the distance can be achieved using not-squared $\ell_2$ norm. Therefore, we extend the inner problem of Eq.(\ref{eq:bi_main}) of the previous model to become:
\begin{eqnarray} \label{eq:bi_ccmtl}
\arg \min_V g(V,e) &=& \arg \min_V  \sum_{i \in [n]} || X_i^\text{tr} w_i -y_i^\text{tr} ||^2 + \frac1{2} \lambda  \sum_{ij \in G} e_{ij}||w_i-w_j||_2 \; ,
\end{eqnarray}
where the regularization term in the inner problem is the non-squared $\ell_2$. This can be efficiently solved using alternating optimization \cite{he_efficient_2019}, by defining the vector of edges' multiplicative weights $l = \mvec{[l_{ij}]}=l(V)$ such that:
\begin{equation} \label{eq:elle}
l_{ij} = 0.5/{ || w_i-w_j ||_2} \; .
\end{equation}
We can now formulate a new optimization problem equivalent to Eq.\ref{eq:bi_ccmtl}
\begin{subequations} \label{eq:bi_ccmtl_sub}
\begin{eqnarray} 
V_e(l) &=& \arg \min_V g(V|e,l) \\
&=& \arg \min_V \sum_{i \in [n]} || X_i^\text{tr} w_i -y_i^\text{tr} ||^2 + \frac1{2} \lambda  \tr { V^T L_{e \circ l} V } + \frac1{4} l^{-\circ }
\end{eqnarray}
\end{subequations}
where $\circ$ is the element-wise product, $L_{e \circ l}$ is the Laplacian matrix whose edge values are the element-wise product of $e$ and $l$ ($e \circ l$), while  $ l^{-\circ } $ is a short notation for the element-wise inverse of $l$. Having fixed $l$, the last term of Eq.\ref{eq:bi_ccmtl_sub} can be ignored, while optimizing the inner problem w.r.t. $V$.
The modified algorithm Alg.\ref{alg:GGMTL} ($\ell_2$-GGMTL), which can also be found in the supplementary material (Alg.\ref{alg:l2-GGMTL}),
uses alternate optimization between the closed-form solution in Eq.\ref{eq:elle} and 
the solution of Eq.\ref{eq:bi_ccmtl_sub} over $V$.

\subsection{Hyper-gradient} \label{sec:hg}
The proposed method (Alg.\ref{alg:GGMTL}) is based on the computation of the hyper-gradient of Eq.\ref{eq:bi_main}. This hyper-gradient has a closed-form as defined by Thm.\ref{th:hp} and can be computed efficiently. 

\begin{theorem} \label{th:hp}
	The hyper-gradient of problem of Eq.\ref{eq:bi_main_out} is
	\begin{eqnarray} \label{eq:hg}
	\dd _{e} f(V_e,e) 	&=&   \xi  e +  \eta \sign{e}  - \gamma \sign{e} \circ  \ln{e} \nonumber \\
	&& - \lambda(V^T\otimes I_m)   (B^T \otimes I_d)  A^{-T}  X^{\text{val},T} ( X^\text{val} V -Y^\text{val})
	\end{eqnarray}
	where $B = [b_{11}b_{11}^T , \dots , b_{nn}b_{nn}^T ] \in R^{n \times nm}$ and $B$ is build with only the $m$ non-zero edges (i.e. $|\{ij| ij \in G\}| = m$).. The other variables are $b_{ij} =(d_i-d_j) \in R^{n \times 1}$, $A = \lambda L_{e} \otimes I_d +  X^TX \in R^{dn \times dn} $ and $V = A^{-1}  X^TY$, $ L_{e} = \sum_{ij} e_{ij} b_{ij}b_{ij}^T$, $V = [w_1^T,\dots,w_n^T]^T \in R^{dn \times 1}$,  $X = \diag{X_1,\dots,X_n} \in R^{Nn \times dn}$, $Y=[y_1,\dots,y_n] \in R^{Nn \times 1}$. $\ln{e}$ is the element wise logarithm of the vector $e$ and $\circ$ is the Hadamard product. $\sign{x}$ is the element-wise sign function of $x$.
\end{theorem}
We notice that $B$ and $A$ in Thm.\ref{th:hp} are sparse matrices. This leads to efficient computation of the hyper-gradient, as shown in Thm.\ref{th:complex}. All proofs are reported in the Supplementary Material (Sec.\ref{sec:proofs}).


\subsection{Closed-form hyper-edges}
Alternative to applying gradient descent methods using the hyper-gradient updates, the optimal edges' values can also be directly computed. 
We compute $e$ (the edge vector) as the solution of $\dd _{e} f(V_e,e) =0$, since the optimal solution has zero hyper-gradient. 
In the case when $\ell_1$ is the only term that has a non-zero weight in Eq.\ref{eq:bi_main_out}, the edge vector $e$ has a closed-form solution as proven in Thm.\ref{th:sol}.

\begin{theorem} \label{th:sol}
	Let suppose $\xi=0, \gamma=0,  \eta \ne 0$, then the hyper-edges of problem of Eq.\ref{eq:bi_main} is the solution of
	\begin{eqnarray} \label{eq:hyper-edges}
	U e & =& v
	\end{eqnarray}
	where $U=(z^T \otimes (E \otimes I_d )) K$,  $K = [\mvec {\diag{d_0} \otimes I_d},\dots, \mvec{\diag{d_{m-1}} \otimes I_d}] \in R^{m^2d^2 \times m}$, where $d_i \in R^{m \times 1}$ is the indicator vector and $u = M^{-1} 1_m$, $v = 1/\eta C  -1/\lambda X^TX u$, $z = (E^T  \otimes I_d ) u $, $M=(V^T\otimes I_m)   (B^T \otimes I_d) \in R^{m \times dn}$  , $C = X^{\text{val},T} ( X^\text{val} V -Y^\text{val}) \in R^{nd \times 1}$ and $E,V,B$ as in Thm.\ref{th:hp}.
\end{theorem}

\subsection{Complexity analysis}
GGMTL algorithm computes the tasks' models kNN graph, whose computational complexity can be reduced from $O(n^2)$ to $O(nd \ln n)$\cite{arya1998optimal} \footnote{or, using Approximate Nearest Neighbour (ANN) methods, to $O(nd)$
\cite{hyvonen2015fast}
}, while one iteration of GGMTL algorithm computes the hyper-gradient. A naive implementation of this step requires inverting a system of dimension $nd \times nd $, whose complexity is $O((dn)^{3})$.  It would thus come to surprise that the actual computational complexity of the GGMTL method is $O(nd \ln n + (nd)^{1.31} +dn^2)$, where the second two terms follow from Thm.\ref{th:complex}, while $O(dn^2)$ is the matrix-vector product which can be performed in parallel. 

\begin{theorem}  \label{th:complex}
The computational complexity of solving hyper-gradient of Thrm.~\ref{th:hp} is $O((nd)^{1.31}+dn^2)$ (or $O((nd)\ln ^c(nd) +dn^2)$, with $c$ constant).  
\end{theorem}




\section{Experimental results} \label{sec:experiment}
We evaluate the performance of GGMTL against four state-of-the-art multi-task learning methodologies (namely MTRL~\cite{zhang2010convex}, MSSL~\cite{gonccalves2016multi}, BSML~\cite{goncalves2019bayesian}, and CCMTL~\cite{he_efficient_2019}) on both synthetic data and real-world applications. Among the four competitors, MTRL learns a graph covariance matrix, MSSL and BSML directly learn a graph precision matrix which can be interpreted as a graph Laplacian. By contrast, CCMTL does not learn task relationship, but uses a fixed $k$-NN graph before learning model parameters\footnote{We performed grid-search hyper-parameter search for all methods}.

\subsection{Measures} 
\subsubsection{Synthetic dataset measure for veracity}
\label{sec:veradicity}
To evaluate the performance of the proposed method on the synthetic dataset, we propose a reformulation of the measures: accuracy, recall and precision by applying the
\L{}ukasiewicz fuzzy T-norm $\top(a,b)= \max(a+b-1,0)$ and T-conorm $\bot(a,b)=\min(a,b)$ \cite{Klement:2000:TN}, where $a,b$ represent truth values from the interval $[0,1]$.
Given the ground truth graph $G_1$ and the predicted graph $G_2$ (on $n$ tasks) with proper adjacency matrices $A_1$ and $A_2$ (i.e., $a^{(1)}_{i,j}, a^{(2)}_{i,j} \in [0,1]$ for all $0\le i,j \le n$), we define:
\begin{eqnarray*}
\text{recall} &= \frac{\sum_{0\le i,j \le n} \top(a^{(1)}_{i,j},a^{(2)}_{i,j})}{\sum_{0\le i,j \le n} a^{(1)}_{i,j} } , \\
\text{precision}  &= \frac{\sum_{0\le i,j \le n} \top(a^{(1)}_{i,j},a^{(2)}_{i,j})}{\sum_{0\le i,j \le n} a^{(2)}_{i,j} } , \\
\text{accuracy}   &= 1- \frac{\sum_{0\le i,j \le n} \oplus(a^{(1)}_{i,j},a^{(2)}_{i,j})}{n^2}
\end{eqnarray*}
s.t $\oplus(a,b) = \top(\bot(a,b),1-\top(a,b))$ is the fuzzy XOR, see \cite{bedregal2009xor}. The definition of the $F_{1}$ score remains unchanged as the harmonic mean of precision and recall. These measures inform about the overlap between a predicted (weighted) graph and a ground truth sparse structure, in a similar way to imbalanced classification. An alternative and less informative approach would be to compute Hamming distance between the two adjacency matrices (ground truth and induced graph), provided they are both binary.
\subsubsection{Regression performance} The generalization performance is measured in terms of the Root Mean Square Error (RMSE) averaged over tasks. 

\begin{table}
	\caption{Results on the synthetic data of each of GGMTL, BMSL, MMSL and CCMTL (KNN) in terms of the measures described in Subsec.
	\ref{sec:veradicity}.}
	\begin{center}
	\begin{tabular}{l|p{.7cm}|l|l|l|l}
		&&BMSL& MSSL&	CCMTL &	GGMTL\\
\hline
\multirow{3}{*}{}accuracy& line& $0.384\pm4.1e^{-2}$&	$0.180\pm1.1e^{-2}$&	$\underline{0.631}\pm2.1e^{-2}$&	$\underline{ \bf 0.648}\pm3.3e^{-2}$\\
& tree& $0.388\pm5.2e^{-2}$&	$0.141\pm1.1e^{-2}$&	$\underline{0.722}\pm1.6e^{-2}$&	$\underline{\bf 0.770}\pm1.2e^{-2}$\\
& star& $\underline{0.581}\pm8.4e^{-2}$&	$\underline{\bf 0.726}\pm2.9e^{-2}$&	$0.405\pm3.1e^{-2}$&	$0.460\pm7.8e^{-2}$\\
\hline
\multirow{3}{*}{}recall& line& $0.688\pm5.6e^{-2}$&	$\underline{\bf 0.958}\pm1.5e^{-2}$&	$0.288\pm4.4e^{-2}$&	$\underline{0.726}\pm1.0e^{-1}$\\
& tree& $0.653\pm6.3e^{-2}$&	$\underline{\bf 0.946}\pm1.3e^{-2}$&	$0.390\pm2.5e^{-2}$&	$\underline{0.790}\pm7.8e^{-2}$\\
& star& $0.318\pm1.3e^{-1}$&	$0.311\pm1.1e^{-1}$&	$\underline{\bf 0.706}\pm7.2e^{-2}$&	$\underline{0.664}\pm1.6e^{-1}$\\
\hline
\multirow{3}{*}{}precision& line& $0.100\pm8.1e^{-4}$&	$\underline{0.100}\pm1.5e^{-4}$&	$0.083\pm7.1e^{-3}$&	$\underline{\bf 0.175}\pm2.7e^{-2}$\\
& tree& $0.065\pm8.4e^{-4}$&	$0.065\pm9.9e^{-5}$&	$\underline{0.092}\pm2.6e^{-3}$&	$\underline{\bf 0.185}\pm1.6e^{-2}$\\
& star& $0.149\pm4.7e^{-2}$&	$\underline{\bf 0.238}\pm5.4e^{-2}$&	$0.176\pm1.2e^{-2}$&	$\underline{ 0.185}\pm3.1e^{-2}$\\
\hline
\multirow{3}{*}{}F1& line& $0.175\pm2.6e^{-3}$&	$\underline{0.182}\pm3.4e^{-4}$&	$0.129\pm1.3e^{-2}$&	$\underline{\bf 0.282}\pm4.3e^{-2}$\\
& tree& $0.117\pm1.9e^{-3}$&	$0.121\pm1.8e^{-4}$&	$\underline{0.149}\pm4.0e^{-3}$&	$\underline{\bf 0.300}\pm2.6e^{-2}$\\
& star& $0.199\pm6.7e^{-2}$&	$0.266\pm7.3e^{-2}$&	$\underline{0.281}\pm2.1e^{-2}$&	$\underline{\bf 0.288}\pm5.e^{-2}$\\
\hline
\multirow{3}{*}{}RMSE& line& $6.968\pm1.5$&	$4.838\pm7.6e^{-1}$& $\underline{4.342}\pm7.21e^{-1}$ & $\underline{4.342}\pm7.21e^{-1}$\\
& tree& $7.444\pm1.4$&	$4.879\pm1.3$&	$\underline{4.207}\pm1.141$ & $\underline{4.207}\pm1.141$	\\
& star& $4.784\pm2.4e^{-1}$&	$1.616\pm2.7e^{-1}$&	$\underline{0.507}\pm2.25e^{-1}$ & $\underline{\bf 0.300}\pm2.11e^{-1}$	\\
\hline
\end{tabular}
	\end{center}
	\label{tbl:synthetic_data}	
\end{table}

\subsection{Synthetic data}

In order to evaluate the veracity of the proposed method, we generate three synthetic datasets where the underlying structure of the relationship among tasks is known. 
Each task $t$ in these datasets is a linear regression task whose output is controlled by the weight vector $\mathbf{w}_t$.
Each input variable $\mathbf{x}$, for task $t$, is generated $i.i.d.$ from an isotropic multivariate Gaussian distribution, and the output is taken by $y =\mathbf{w}_t^T\mathbf{x}+\epsilon$, where $\epsilon\sim\mathcal{N}(0,1)$.

\begin{figure}
	\centering
	\begin{subfigure}{.34\textwidth}
		\centering
		\includegraphics[width=1.\linewidth,trim=0 0cm 0 1cm, clip]{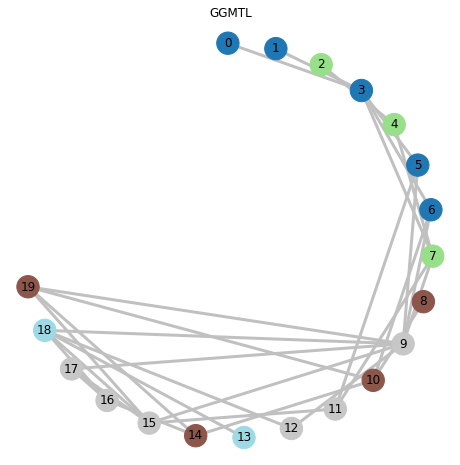}
		\caption{GGMTL (line)}
		\label{fig:GGMTL_line}
	\end{subfigure}%
	\begin{subfigure}{.34\textwidth}
		\centering
		\includegraphics[width=1.\linewidth,trim=0 0cm 0 1cm, clip]{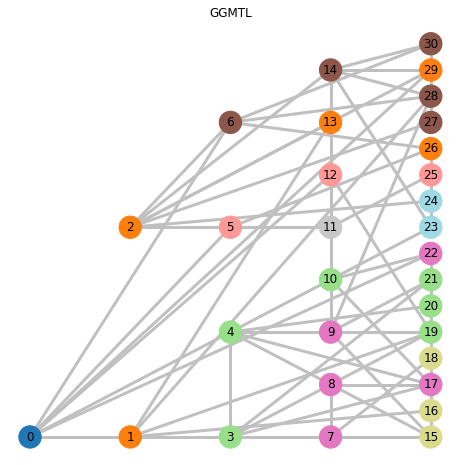}
		\caption{GGMTL (tree)}
		\label{fig:GGMTL_tree}
	\end{subfigure}%
	\begin{subfigure}{.34\textwidth}
		\centering
		\includegraphics[width=1.\linewidth,trim=0 0cm 0 1cm, clip]{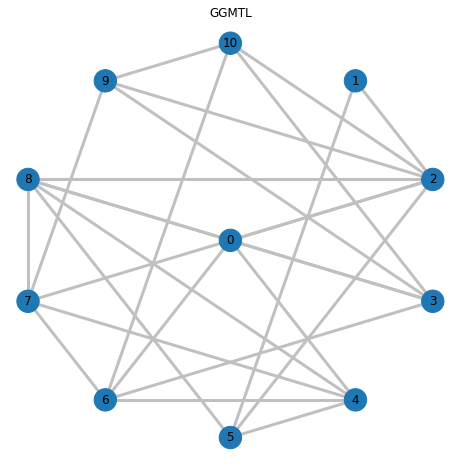}
		\caption{GGMTL (star)}
		\label{fig:GGMTL_star}
	\end{subfigure}%
	\\
	\begin{subfigure}{.34\textwidth}
		\centering
		\includegraphics[width=1.\linewidth,trim=0 0cm 0 1cm, clip]{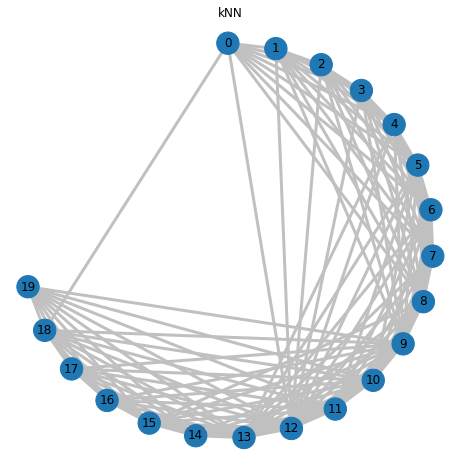}
		\caption{kNN (line)}
		\label{fig:kNN_line}
	\end{subfigure}%
	\begin{subfigure}{.34\textwidth}
		\centering
		\includegraphics[width=1.\linewidth,trim=0 0cm 0 1cm, clip]{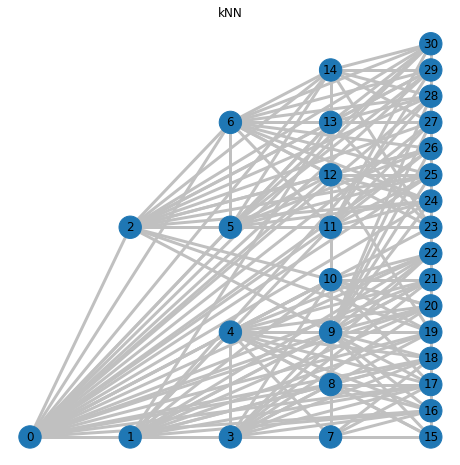}
		\caption{kNN (tree)}
		\label{fig:kNN_tree}
	\end{subfigure}%
	\begin{subfigure}{.34\textwidth}
		\centering
		\includegraphics[width=1.\linewidth,trim=0 0cm 0 1cm, clip]{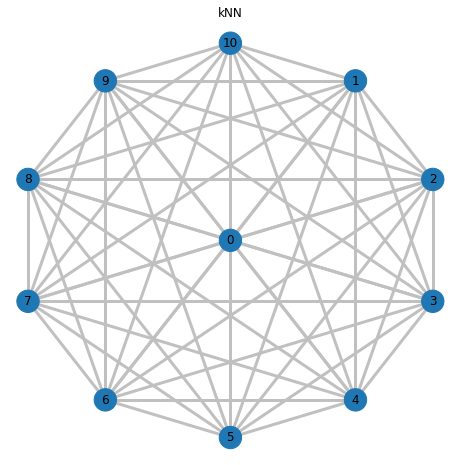}
		\caption{kNN (star)}
		\label{fig:kNN_star}
	\end{subfigure}%
	
	\caption{The discovered graphs by GGMTL and k-NN on the synthetic datasets: \texttt{Line}, \texttt{Tree} and \texttt{Star}.}
	\label{fig:synth}
\end{figure}

The first dataset \texttt{Line} mimics the structure of a line, where each task is generated with an overlap to its predecessor task. This dataset contains $20$ tasks of $30$ input dimensions. The coefficient vector for tasks $t$ is $\mathbf{w}_t =\mathbf{w}_{t-1} + 0.1 \mathbf{u}_{30} \odot \mathbf{b}_{30}$, where $\odot$ denotes the pointwise product, $\mathbf{u}_{30}$ is a $30$-dimensional random vector with each element uniformly distributed between $[0,1]$,  $\mathbf{b}_{30}$ is also a $30$-dimensional binay vector whose elements are Bernoulli distributed with $p=0.7$, and $\mathbf{w}_0\sim\mathcal{N}(\mathbf{1},\mathbf{I}_{30})$.

The tasks of the second dataset \texttt{Tree} are created in a hierarchical manner simulating a tree structure such that 
$\mathbf{w}_t =\mathbf{w}_{t^\prime} + 0.1 \mathbf{u}_{30}\odot \mathbf{b}_{30}$, where 
$\mathbf{w}_{t^\prime}$ is coefficient vector of the parent task ($t^\prime= \floor*{(t-1)/2}$), and $\mathbf{w}_0\sim\mathcal{N}(\mathbf{1},\mathbf{I}_{30})$ is for the root task. In order to create a proper binary tree, we generate $31$ tasks ($30$-dimensional) which creates a tree of five levels.

The distribution of the third dataset's tasks takes a star-shaped structure, hence called \texttt{Star}. The Coefficient vector of each task $t$ is randomly created ($\mathbf{w}_t\sim\mathcal{N}(\mathbf{1},\mathbf{I}_{20}$) for $t\in \{1,\dots,10\}$), and the center one is a mixture of them
$\mathbf{w}_0 = \sum_{t\in\{1,\dots,10\}} \mathbf{w}_t \cdot (e_{2t-1}+e_{2t})$, where $e_i \in \mathbb{R}^{T}$ is an indicator vector with the $i$th element set to $1$ and the others to $0$.
We evaluate the performance of our method in comparison to the other methods on two aspects: \textit{(i)} the ability to learn graphs that recover hidden sparse structures, and \textit{(ii)} the generalization performance. For generalization error we use Root Mean Square Error (RMSE).

Tab.\ref{tbl:synthetic_data} depicts the results of comparing the graphs learned by GGMTL with those of CCMTL, and the covariance matrices of MSSL and BMSL when considered as adjacency matrices, after few adjustments
\footnote{A negative correlation is considered as a missing edge between tasks, hence, negative entries are set to zero. Besides, we normalize each matrix by division over the largest entry after setting the diagonal to zero}. 
It is apparent that GGMTL always achieves the best accuracy except on the \texttt{Star} dataset when MSSL performs best in terms of accuracy; this occurs only because MSSL predicts an extremely sparse matrix leading to poor recall, precision and $F_{1}$ score. Moreover, GGMTL has always the best $F_{1}$ score achieved by correctly predicting the right balance between edges (with 2nd best recall) and sparseness (always best precsion), thus, leading to correctly interpreting and revealing the latent structure of task relations. Besides the quantitative measures, interpretability is also confirmed qualitatively in Fig.\ref{fig:synth} where the discovered edges reveal to a large extent the ground truth structures. The figure also plots the k-NN graph next to that of GGMTL, this shows how graphs of GGMTL pose a refinement of those of k-NN by removing misplaced edges while still maintaining the relevant ones among tasks.
Finally, Tab.\ref{tbl:synthetic_data} also shows that GGMTL commits the smallest generalization error in terms of RMSE with a large margin to BMSL and MSSL.

\begin{table}
\centering
\small
\caption{RMSE (mean$\pm$std) on Parkinson's disease data set over $10$ independent runs on various train/test ratios $r$. Best two performances underlined.}\label{Tab:parkinson}
 \begin{tabular}{lllllll}
\toprule
split & MTRL & MSSL & BMSL & CCMTL & \textbf{GGMTL}  \\
\midrule
$r=0.3$ & $4.147\pm3.038$ & $1.144\pm0.007$ &$1.221 \pm 0.11$ & $\underline{1.037}\pm0.013$ & $\underline{\bf 1.037}\pm0.012$  \\
$r=0.4$ & $3.202\pm2.587$ & $1.129\pm0.011$ &$1.150 \pm 0.1$ & $\underline{\bf 1.017}\pm0.008$ & $\underline{1.019}\pm0.01$ \\
$r=0.5$ & $1.761\pm0.85$ & $1.130\pm0.009$ &$1.110 \pm 0.085$ & $\underline{1.010}\pm0.008$ & $\underline{1.010}\pm0.008$  \\
$r=0.6$ & $1.045\pm0.05$ & $1.123\pm0.013$ &$1.068 \pm  0.036$ & $\underline{\bf 0.998}\pm0.017$ & $\underline{1.000}\pm0.017$  \\
\bottomrule
\end{tabular}
\end{table}

\begin{figure}
	\centering
		\begin{subfigure}{.5\textwidth}
		\centering
		\includegraphics[width=1.\linewidth,trim=0 0cm 0 1cm, clip]{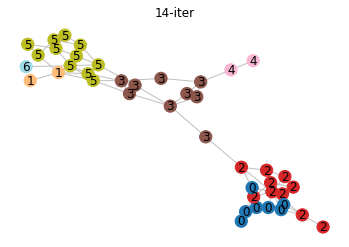}
		\caption{GGMTL (Parkinson)}
		\label{fig:GGMTL_parkinsons}
	\end{subfigure}%
	\begin{subfigure}{.5\textwidth}
		\centering
		\includegraphics[width=1.\linewidth,trim=0 0cm 0 1cm, clip]{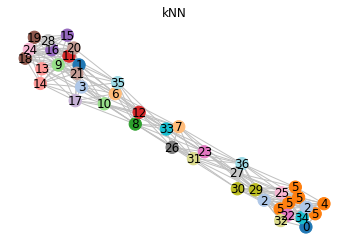}
		\caption{kNN (Parkinson)}
		\label{fig:kNN_parkinsons}
	\end{subfigure}	
	\\
		\begin{subfigure}{.5\textwidth}
		\centering
		\includegraphics[width=1.\linewidth,trim=0 0cm 0 1cm, clip]{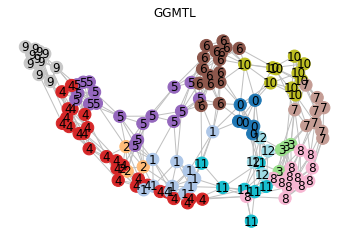}
		\caption{GGMTL (School)}
		\label{fig:GGMTL_school}
	\end{subfigure}%
	\begin{subfigure}{.5\textwidth}
		\centering
		\includegraphics[width=1.\linewidth,trim=0 0cm 0 1cm, clip]{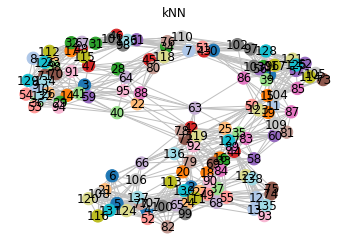}
		\caption{kNN (School)}
		\label{fig:kNN_school}
	\end{subfigure}%
	\\
	\begin{subfigure}{.5\textwidth}
		\centering
		\includegraphics[width=1.\linewidth,trim=0 0cm 0 1cm, clip]{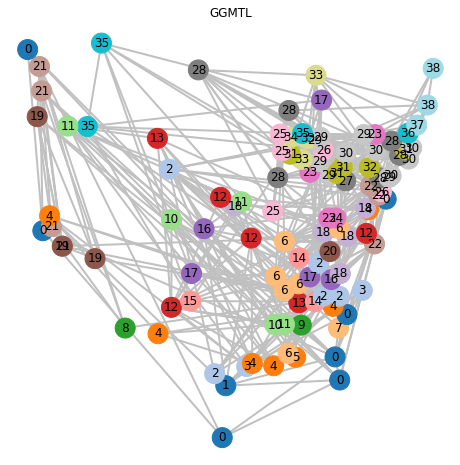}
		\caption{GGMTL (temperature)}
		\label{fig:GGMTL_temp}
	\end{subfigure}%
	\begin{subfigure}{.5\textwidth}
		\centering
		\includegraphics[width=1.\linewidth,trim=0 0cm 0 1cm, clip]{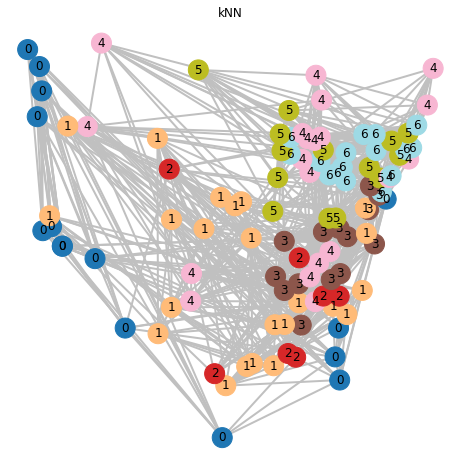}
		\caption{kNN (temperature)}
		\label{fig:kNN_temp}
	\end{subfigure}	
	
	\caption{The discovered graphs by GGMTL and k-NN on the \texttt{Parkinson}, \texttt{School} and \texttt{South} datasets.}
	\label{fig:real_data}
\end{figure}

\subsection{Real-world applications}
\subsubsection{Parkinson's disease assessment.}
\texttt{Parkinson} is a benchmark multi-task regression dataset \footnote{\url{https://archive.ics.uci.edu/ml/datasets/parkinsons+telemonitoring}}, comprising a range of biomedical voice measurements taken from $42$ patients with early-stage Parkinson's disease. For each patient, the goal is to predict the motor Unified Parkinson's Disease Rating Scale (UPDRS) score based $18$-dimensional record: age, gender, and $16$ jitter and shimmer voice measurements. We treat UPDRS prediction for each patient as a task, resulting in $42$ tasks and $5,875$ observations in total.
We compare the generalization performance of GGMTL with that of the other baselines when different ratios of the data is used for training, ratio $r \in\{0.3,0.4,0.5,0.6\}$.
The results depicted in Tab.\ref{Tab:parkinson} show that GGMTL performance is close to that of CCMTL, outperforming MSSL and BMSL.
However, when plotting the learned graphs, see Fig.\ref{fig:real_data}(a) and Fig.\ref{fig:real_data}(b), GGMTL clearly manage to separate patients into a few distinct groups unlike the heavily connected k-NN graph used by CCMTL. Interestingly, these groups are easily distinguished by Markov Clustering \cite{van2000graph} when applied on the learned graph; this very same procedure fails to distinguish reasonable clusters when applied on the K-NN graph (35 clusters were discovered with only one task in each, and one cluster with five tasks).



\subsubsection{Exam score prediction}
\begin{table}
	\centering
	\small
	\caption{RMSE (mean$\pm$std) on School data set over $10$ independent runs on various train/test ratios $r$. Best two performances underlined.}\label{Tab:school}
	\begin{tabular}{c|lllll}
		\toprule
		& MTRL & MSSL & BMSL & CCMTL & \textbf{GGMTL}  \\
		\midrule
$r=0.2$ & ${11.276}\pm0.103$   & $11.727 \pm 0.137$ & $10.430 \pm 0.056$ & $\underline{\bf 10.129}\pm0.038$ & $\underline{10.137}\pm0.034$  \\
$r=0.3$ & ${10.761}\pm0.045$  & $11.060 \pm 0.045$ & $10.223 \pm 0.044$ & $\underline{10.078}\pm0.038$ & $\underline{\bf 10.072}\pm0.037$  \\
$r=0.4$ & ${10.440}\pm0.069$  & $10.632 \pm 0.076$ & $10.084 \pm 0.067$ & $\underline{10.059}\pm0.067$ & $\underline{\bf 10.056}\pm0.068$  \\
$r=0.5$ & ${10.267}\pm0.065$  & $10.437 \pm 0.040$ & $10.048 \pm 0.043$ & $\underline{\bf 9.994}\pm0.052$ & $\underline{9.996}\pm0.051$  \\
		\bottomrule
	\end{tabular}
\end{table}
\texttt{School} is a classical benchmark dataset in Multi-task regression \cite{argyriou2007multi,kumar2012learning,zhang2014regularization}; it consists of examination scores of $15,362$ students from $139$ schools in London. Each school is considered as a task and the aim is to predict the exam scores for all the students. The school dataset is available in the Malsar package \cite{zhou2011malsar}.

Tab.\ref{Tab:school} reports the RMSE on the School dataset. It is noticeable that both GGMTL and CCMTL have similar but dominating performance over the other methods. As with the Parkinson data, Fig.\ref{fig:real_data}(c) and Fig.\ref{fig:real_data}(d) compare the graphs induced by GGMTL and CCMTL(k-NN), and the results of applying Markov clustering on their nodes. These figures show again that graphs induced by GGMTL are easier to interpret and lead to well separated clusters with only few intercluster edges.


\subsubsection{Temperature forecasting in U.S.}
\begin{table}[!hbpt]
	\centering
	\small
	\caption{RMSE (mean$\pm$std) on US America Temperature data. 
	The best two performances are underlined.}\label{Tab:temperature}
	\begin{tabular}{ccccccc}
		\toprule
		Hours & MTRL & MSSL & BMSL & CCMTL & \textbf{GGMTL}  \\
		\midrule
		8 & NA  & $1.794$ & $0.489 $ & $\underline{0.101}$ & $\underline{ 0.101}$  \\
		16 & NA  & $1.643$ & $0.606 $ & $\underline{0.0975}$ & $\underline{ \bf  0.0973}$  \\
		\bottomrule
	\end{tabular}
\end{table}
The \texttt{Temperature} dataset\footnote{\url{https://www.ncdc.noaa.gov/data-access/land-based-station-data/land-based-datasets/climate-normals/1981-2010-normals-data}} contains hourly temperature data for major cities in the United States, collected from $n=109$ stations for $8759$ hours in $2010$. Data is cleaned and manipulated as described in ~\cite{hua2020online}. 
The temperature forecasting with a horizon of $8$ or $16$ hours in advance at each station is model as a task. 
We select the first $20\%$ observations (roughly $5$ weeks) to train and left the remaining $80\%$ observations for test. 
We use previous $30$ hours of temperature as input to the model.
Tab.\ref{Tab:temperature} reports the RMSE of the methods. Learning the graph structure using GGMTL does not impact performance in term of regression error.  
Fig.\ref{fig:real_data}[e-f] shows node clustering on the graph learned on \texttt{Temerature} dataset, where the number of edges is reduced by $60\%$ using GGMTL.

\section{Conclusions}
In this work, we present a novel formulation of joint multi-task and graph structure learning as a bilevel problem, and propose an efficient method for solving it based on a closed-form of hyper-gradient. We also show the interpretability property of the proposed method on synthetic and real world datasets. We additionally analyze the computational complexity of the proposed method. 

\bibliographystyle{plain}
\bibliography{IEEEabrv,references_a,references_b}

\begin{thebibliography}{10}

\bibitem{argyriou2007multi}
Andreas Argyriou, Theodoros Evgeniou, and Massimiliano Pontil.
\newblock Multi-task feature learning.
\newblock In {\em NIPS}, pages 41--48, 2007.

\bibitem{argyriou2008convex}
Andreas Argyriou, Theodoros Evgeniou, and Massimiliano Pontil.
\newblock Convex multi-task feature learning.
\newblock {\em Machine Learning}, 73(3):243--272, 2008.

\bibitem{arya1998optimal}
Sunil Arya, David~M Mount, Nathan~S Netanyahu, Ruth Silverman, and Angela~Y Wu.
\newblock An optimal algorithm for approximate nearest neighbor searching fixed
  dimensions.
\newblock {\em Journal of the ACM (JACM)}, 45(6):891--923, 1998.

\bibitem{bedregal2009xor}
Benjam{\'\i}n~C Bedregal, Renata~HS Reiser, and Gra{\c{c}}aliz~P Dimuro.
\newblock Xor-implications and e-implications: classes of fuzzy implications
  based on fuzzy xor.
\newblock {\em Electronic notes in theoretical computer science}, 247:5--18,
  2009.

\bibitem{belkin_laplacian_2002}
Mikhail Belkin and Partha Niyogi.
\newblock Laplacian {Eigenmaps} and {Spectral} {Techniques} for {Embedding} and
  {Clustering}.
\newblock In {\em Advances in {Neural} {Information} {Processing} {Systems}
  14}, pages 585--591. MIT Press, 2002.

\bibitem{chen2011integrating}
Jianhui Chen, Jiayu Zhou, and Jieping Ye.
\newblock Integrating low-rank and group-sparse structures for robust
  multi-task learning.
\newblock In {\em KDD}, pages 42--50. ACM, 2011.

\bibitem{colson2007overview}
Beno{\^\i}t Colson, Patrice Marcotte, and Gilles Savard.
\newblock An overview of bilevel optimization.
\newblock {\em Annals of operations research}, 153(1):235--256, 2007.

\bibitem{deng2017situation}
Dingxiong Deng, Cyrus Shahabi, Ugur Demiryurek, and Linhong Zhu.
\newblock Situation aware multi-task learning for traffic prediction.
\newblock In {\em ICDM}, pages 81--90. IEEE, 2017.

\bibitem{dong2016learning}
Xiaowen Dong, Dorina Thanou, Pascal Frossard, and Pierre Vandergheynst.
\newblock Learning laplacian matrix in smooth graph signal representations.
\newblock {\em IEEE Transactions on Signal Processing}, 64(23):6160--6173,
  2016.

\bibitem{flamary2014learning}
R{\'e}mi Flamary, Alain Rakotomamonjy, and Gilles Gasso.
\newblock Learning constrained task similarities in graphregularized multi-task
  learning.
\newblock {\em Regularization, Optimization, Kernels, and Support Vector
  Machines}, page 103, 2014.

\bibitem{franceschi_bilevel_2018}
Luca Franceschi, Paolo Frasconi, Saverio Salzo, Riccardo Grazzi, and
  Massimilano Pontil.
\newblock Bilevel {Programming} for {Hyperparameter} {Optimization} and
  {Meta}-{Learning}.
\newblock {\em arXiv:1806.04910 [cs, stat]}, June 2018.

\bibitem{frecon2018bilevel}
Jordan Frecon, Saverio Salzo, and Massimiliano Pontil.
\newblock Bilevel learning of the group lasso structure.
\newblock In {\em Advances in Neural Information Processing Systems}, pages
  8301--8311, 2018.

\bibitem{goncalves2019bayesian}
Andre Goncalves, Priyadip Ray, Braden Soper, David Widemann, Mari Nyg{\aa}rd,
  Jan~F Nyg{\aa}rd, and Ana~Paula Sales.
\newblock Bayesian multitask learning regression for heterogeneous patient
  cohorts.
\newblock {\em Journal of Biomedical Informatics: X}, 4:100059, 2019.

\bibitem{goncalves_multi-task_2014}
Andre~R. Goncalves, Puja Das, Soumyadeep Chatterjee, Vidyashankar Sivakumar,
  Fernando~J. Von~Zuben, and Arindam Banerjee.
\newblock Multi-task {Sparse} {Structure} {Learning}.
\newblock {\em Proceedings of the 23rd ACM CIKM '14}, 2014.

\bibitem{gonccalves2016multi}
Andr{\'e}~R Gon{\c{c}}alves, Fernando~J Von~Zuben, and Arindam Banerjee.
\newblock Multi-task sparse structure learning with gaussian copula models.
\newblock {\em The Journal of Machine Learning Research}, 17(1):1205--1234,
  2016.

\bibitem{goncalves_multi-task_2016}
Andre~R Goncalves, Fernando J~Von Zuben, and Arindam Banerjee.
\newblock Multi-task {Sparse} {Structure} {Learning} with {Gaussian} {Copula}
  {Models}.
\newblock {\em Journal of Machine Learning Research 17 (2016)}, page~30, 2016.

\bibitem{guidotti_survey_2018}
Riccardo Guidotti, Anna Monreale, Salvatore Ruggieri, Franco Turini, Fosca
  Giannotti, and Dino Pedreschi.
\newblock A {Survey} of {Methods} for {Explaining} {Black} {Box} {Models}.
\newblock {\em ACM Computing Surveys}, 51(5):1--42, August 2018.

\bibitem{han2015learning}
Lei Han and Yu~Zhang.
\newblock Learning multi-level task groups in multi-task learning.
\newblock In {\em AAAI}, volume~15, pages 2638--2644, 2015.

\bibitem{he_efficient_2019}
Xiao He, Francesco Alesiani, and Ammar Shaker.
\newblock Efficient and {Scalable} {Multi}-task {Regression} on {Massive}
  {Number} of {Tasks}.
\newblock In {\em The {Thirty}-{Third} {AAAI} {Conference} on {Artificial}
  {Intelligence} ({AAAI}-19)}, 2019.

\bibitem{hua2020online}
Fei Hua, Roula Nassif, C{\'e}dric Richard, Haiyan Wang, and Ali~H Sayed.
\newblock Online distributed learning over graphs with multitask graph-filter
  models.
\newblock {\em IEEE Transactions on Signal and Information Processing over
  Networks}, 6:63--77, 2020.

\bibitem{huang_sparse_2019}
Shuai Huang and Trac~D. Tran.
\newblock Sparse {Signal} {Recovery} via {Generalized} {Entropy} {Functions}
  {Minimization}.
\newblock {\em IEEE Transactions on Signal Processing}, 67(5):1322--1337, March
  2019.

\bibitem{hyvonen2015fast}
Ville Hyv{\"o}nen, Teemu Pitk{\"a}nen, Sotiris Tasoulis, Elias
  J{\"a}{\"a}saari, Risto Tuomainen, Liang Wang, Jukka Corander, and Teemu
  Roos.
\newblock Fast k-nn search.
\newblock {\em arXiv:1509.06957}, 2015.

\bibitem{intelligence_policy_2019}
High-Level Expert Group on~Artificial Intelligence.
\newblock Policy and investment recommendations for trustworthy {AI}.
\newblock June 2019.
\newblock Publisher: European Commission Type: Article; Article/Report.

\bibitem{jacob2009clustered}
Laurent Jacob, Jean-philippe Vert, and Francis~R Bach.
\newblock Clustered multi-task learning: A convex formulation.
\newblock In {\em NIPS}, pages 745--752, 2009.

\bibitem{jenni2018deep}
Simon Jenni and Paolo Favaro.
\newblock Deep bilevel learning.
\newblock In {\em Proceedings of the European Conference on Computer Vision
  (ECCV)}, pages 618--633, 2018.

\bibitem{Klement:2000:TN}
Erich~Peter Klement, Radko Mesiar, and Endre Pap.
\newblock {\em Triangular Norms}.
\newblock Kluwer Academic Publishers, Dordrecht, The Netherlands, 2000.

\bibitem{kumar2012learning}
Abhishek Kumar and Hal Daume~III.
\newblock Learning task grouping and overlap in multi-task learning.
\newblock {\em ICML}, 2012.

\bibitem{kunapuli2008bilevel}
Gautam Kunapuli.
\newblock {\em A bilevel optimization approach to machine learning}.
\newblock PhD thesis, PhD thesis, Rensselaer Polytechnic Institute, 2008.

\bibitem{li2018multi}
Limin Li, Xiao He, and Karsten Borgwardt.
\newblock Multi-target drug repositioning by bipartite block-wise sparse
  multi-task learning.
\newblock {\em BMC systems biology}, 12, 2018.

\bibitem{lipton_mythos_2018}
Zachary~C. Lipton.
\newblock The mythos of model interpretability.
\newblock {\em Communications of the ACM}, 61(10):36--43, September 2018.

\bibitem{liu_multi-task_2018}
Pengfei Liu, Jie Fu, Yue Dong, Xipeng Qiu, and Jackie Chi~Kit Cheung.
\newblock Multi-task {Learning} over {Graph} {Structures}.
\newblock {\em arXiv:1811.10211 [cs]}, November 2018.

\bibitem{lundberg_unified_2017}
Scott Lundberg and Su-In Lee.
\newblock A {Unified} {Approach} to {Interpreting} {Model} {Predictions}.
\newblock {\em arXiv:1705.07874 [cs, stat]}, November 2017.
\newblock arXiv: 1705.07874.

\bibitem{murugesan_adaptive_2016}
Keerthiram Murugesan, Jaime Carbonell, Hanxiao Liu, and Yiming Yang.
\newblock Adaptive {Smoothed} {Online} {Multi}-{Task} {Learning}.
\newblock page~11, 2016.

\bibitem{obozinski2010joint}
Guillaume Obozinski, Ben Taskar, and Michael~I Jordan.
\newblock Joint covariate selection and joint subspace selection for multiple
  classification problems.
\newblock {\em Statistics and Computing}, 20(2):231--252, 2010.

\bibitem{petersen_matrix_2008}
Kaare~Brandt Petersen and Michael~Syskind Pedersen.
\newblock {\em The {Matrix} {Cookbook}}.
\newblock February 2008.

\bibitem{ribeiro_why_2016}
Marco~Tulio Ribeiro, Sameer Singh, and Carlos Guestrin.
\newblock "{Why} {Should} {I} {Trust} {You}?": {Explaining} the {Predictions}
  of {Any} {Classifier}.
\newblock {\em arXiv:1602.04938 [cs, stat]}, August 2016.

\bibitem{ruder2017overview}
Sebastian Ruder.
\newblock An overview of multi-task learning in deep neural networks.
\newblock {\em arXiv preprint arXiv:1706.05098}, 2017.

\bibitem{saha_online_2011}
Avishek Saha, Piyush Rai, Hal~Daume Iii, and Suresh Venkatasubramanian.
\newblock Online {Learning} of {Multiple} {Tasks} and {Their} {Relationships}.
\newblock page~9, 2011.

\bibitem{spielman_solving_2004}
Daniel~A. Spielman and Shang-Hua Teng.
\newblock Solving {Sparse}, {Symmetric}, {Diagonally}-{Dominant} {Linear}
  {Systems} in {Time} \${O} (m{\textasciicircum}\{1.31\})\$.
\newblock {\em arXiv:cs/0310036}, March 2004.

\bibitem{spielman_nearly-linear_2012}
Daniel~A. Spielman and Shang-Hua Teng.
\newblock Nearly-{Linear} {Time} {Algorithms} for {Preconditioning} and
  {Solving} {Symmetric}, {Diagonally} {Dominant} {Linear} {Systems}.
\newblock {\em arXiv:cs/0607105}, September 2012.

\bibitem{van2000graph}
Stijn~Marinus Van~Dongen.
\newblock {\em Graph clustering by flow simulation}.
\newblock PhD thesis, 2000.

\bibitem{yuan2006model}
Ming Yuan and Yi~Lin.
\newblock Model selection and estimation in regression with grouped variables.
\newblock {\em Journal of the Royal Statistical Society: Series B (Statistical
  Methodology)}, 68(1):49--67, 2006.

\bibitem{Zhang2017Survey}
Yu~Zhang and Qiang Yang.
\newblock A survey on multi-task learning.
\newblock {\em arXiv preprint arXiv:1707.08114v2}, 2017.

\bibitem{zhang2010convex}
Yu~Zhang and Dit-Yan Yeung.
\newblock A convex formulation for learning task relationships in multi-task
  learning.
\newblock In {\em Proceedings of the Twenty-Sixth Conference on UAI}, pages
  733--742, 2010.

\bibitem{zhang2014regularization}
Yu~Zhang and Dit-Yan Yeung.
\newblock A regularization approach to learning task relationships in multitask
  learning.
\newblock {\em ACM Transactions on TKDD}, 8, 2014.

\bibitem{zhou2011malsar}
Jiayu Zhou, Jianhui Chen, and Jieping Ye.
\newblock Malsar: Multi-task learning via structural regularization.
\newblock {\em Arizona State University}, 21, 2011.

\end{thebibliography}

\newpage
\appendix
\section{Supplementary material}
This Supplementary material contains:
\begin{itemize}
    \item The proofs of the Theorems stated in the main part (Annex \ref{sec:proofs});
    \item the two versions of the algorithm: $\ell_2^2-GGMTL$ and $\ell_2-GGMTL$ (Annex \ref{sec:GGMTL})
    \item additional datasets on the America climate. This dataset shows also the computational efficiency of the method that we were able to apply to moderate large size dataset (i.e. $490$ nodes of the North America dataset) (Annex \ref{sec:other_datasets})  
\end{itemize}

\subsection{Proof of theorems} \label{sec:proofs}
\begin{proof}[Theorem \ref{th:hp}]
Since $V = [w_1^T,\dots,w_n^T]^T \in R^{dn \times 1}$ the vector of all models,  we define $X = \diag{X_1,\dots,X_n} \in R^{Nn \times dn}$ the block matrix of input and $Y=[y_1,\dots,y_n] \in R^{Nn \times 1}$ the vector of the output, the solution of the lower problem is
$$
(\lambda L_e \otimes I_d +  X^TX) V_e = X^TY
$$
or
\begin{eqnarray} \label{eq:sol_inner}
V_e = A^{-1} X^TY 
\end{eqnarray}
where we define the auxiliary matrix $A = \lambda L_{e} \otimes I_d +  X^TX$
%

Using the Sherman Morrison formula
$$
(A+uv^T)^{-1} = A^{-1} - \frac{A^{-1} uv^T A^{-1}}{ 1 + v^T A^{-1}u }
$$
and defining $b = \sqrt {\lambda \delta e_{ij}} (d_i -d_j) \otimes I_d$
we can write the increment $d_{ij} V$ on the edge $ij$ of $\delta e_{ij}$ as the difference of the models,
\begin{eqnarray}
d_{ij} V &=& (bb^T +A)^{-1} X^TY - V\\
&=& A^{-1}X^TY - \frac{A^{-1} bb^TA^{-1} }{1+b^TA^{-1} b}X^TY -A^{-1}X^TY \\
&=&- \frac{A^{-1} bb^TA^{-1} }{1+b^TA^{-1} b}X^TY
\end{eqnarray}
if $\delta e_{ij} \to 0$ then $b^TA^{-1} b \to 0$
\begin{eqnarray}
\nabla_{e_{ij}}  V(e)  
&=&   - \lambda  A^{-1}   ((d_i-d_j) \otimes I_d ) ((d_i-d_j)^T \otimes I_d) A^{-1}  X^TY  \nonumber \\
&=& - \lambda  A^{-1}  ((d_i-d_j)(d_i-d_j)^T \otimes I_d)  V 
\end{eqnarray}
since $\nabla_{e_{ij}}  V(e) = \lim_{\delta  e_{ij} \to 0} \frac{d_{ij} V }{\delta  e_{ij}}$ (see also the Thm.\ref{th:directional}).
Thus we have the shape of parameter gradient with respect to the hyper-parameters

\begin{eqnarray}
\nabla_{e^T}  V(e) & = &- \lambda A^{-1} (\underbrace{[b_{11}b_{11}^T , \dots , b_{nn}b_{nn}^T ]}_{B: ~ m ~ \text{items}} \otimes I_d) (V\otimes I_m) \\ 
&=& - \lambda A^{-1} (B \otimes I_d) (V\otimes I_m)
\end{eqnarray}

where $B_{ij} = b_{ij}b_{ij}^T =(d_i-d_j)(d_i-d_j)^T$, $A = \lambda L_{e} \otimes I_d +  X^TX$, $V = A^{-1}  X^TY$ and  $\nabla_{e^T}   V(e)  \in R^{dn \times m}$.



We have that
$$
\nabla_{w_i} f(V,e) =  X_i^{\text{val},T} ( X_i^\text{val} w_i -y_i^\text{val})
$$
or for all models
$$
\nabla_{V} f(V,e) =  X^{\text{val},T} ( X^\text{val} V -Y^\text{val})
$$

The general expression of hyper-gradient or total derivative is given by
\begin{equation}
\dd _{\lambda} f(w_\lambda,\lambda) = \nabla_\lambda f(w_\lambda,\lambda)+  \nabla_\lambda w_\lambda \nabla_w f(w_\lambda,\lambda)
\end{equation}
where $\nabla_\lambda$ is the partial derivative, while $\dd_\lambda$ is the total derivative, and we assume $w,\lambda$ the parameters and hyper-parameters. 
We notice that $\nabla_e H(e) = - \nabla_e \sum_{ij \in G} (|e_{ij}| \ln{ |e_{ij}|}- |e_{ij}|) =  - \sign{e} \circ \ln{ e} $, where $\circ$ is the Hadamard product.
We can now write the hyper-gradient for our problem
\begin{eqnarray} \label{eq:hg_proof}
\dd _{e} f(V_e,e)
&=& \nabla_e f(V_e,e)+  \nabla_e V_e \nabla_V f(V_e,e) \\
&=&  \xi  e +  \eta \sign{e} - \gamma \sign{e} \circ  \ln{e} \nonumber \\
&& - \lambda(V^T\otimes I_m)   (B^T \otimes I_d)  A^{-T}  X^{\text{val},T} ( X^\text{val} V -Y^\text{val})
\end{eqnarray}
where $B = [b_{11}b_{11}^T , \dots , b_{nn}b_{nn}^T ] \in R^{n \times nm}$. The other variables previously defined are $b_{ij} =(e_i-e_j) \in R^{n \times 1}$, $A = \lambda L_{e} \otimes I_d +  X^TX \in R^{dn \times dn} $ and $V = A^{-1}  X^TY$, $ L_{e} = \sum_l e_{ij} b_{l}$, $V = [w_1^T,\dots,w_n^T]^T \in R^{dn \times 1}$,  $X = \diag{X_1,\dots,X_n} \in R^{Nn \times dn}$, $Y=[y_1,\dots,y_n] \in R^{Nn \times 1}$  If $e\ge0$ then $\sign{e}=1_m$. 
\end{proof}

\begin{theorem}  \label{th:SSD}
The matrix $A$ of Thm.~\ref{th:hp} defines a Sparse, Symmetric, Diagonally-Dominant (SDD) linear system.
\end{theorem}

\begin{proof}[Theorem \ref{th:SSD}]
The property comes from inspecting the component of  
$A = \lambda L_{e} \otimes I_d +  X^TX $. The first term is a block diagonal matrix, whose block are the Laplacian matrix of the graph. This matrix is symmetric and sparse. The second element is a block diagonal matrix whose blocks are $X_i^TX_i \in R^{d \times d}$, that are symmetric matrices. Thus the system $AX=Y$ is a SDD linear system. 
\end{proof}

\begin{theorem}  \label{th:directional}
	The directional derivative of $F(A)=A^{-1}$ along $B$ is $-A^{-1}BA^{-1}$.
\end{theorem}

\begin{proof}[Theorem \ref{th:directional}]
	According to \cite{petersen_matrix_2008} page 19, Ch.3.4 , Eq. 167, we have that $(A+hB)^{-1} \approx A^{-1} - hA^{-1}BA^{-1}$ when $h \ll A,B$. When now consider the following directional derivative of the function
	$F: R^{n \times n} \to R^{n \times n} : A \to A^{-1}$ or $F(A)=A^{-1}$, when then compute
	$F(A+hB)-F(A) =(A+hB)^{-1}-A^{-1} \approx A^{-1} - hA^{-1}BA^{-1} -A^{-1} = - hA^{-1}BA^{-1}$. It follows that the directional derivative is $-A^{-1}BA^{-1}$.
\end{proof}

\begin{proof}[Theorem \ref{th:sol}]
We notice that by setting $\dd _{e} f(V_{e^{(t)}},e^{(t)}) =0$ we can compute the optimal $e$ in closed-form.
If we assume $e \ge 0$ and $\xi =0 $, we have
$$
1_m - \lambda' (V^T\otimes I_m)   (B^T \otimes I_d)  A^{-T}  X^{\text{val},T} ( X^\text{val} V -Y^\text{val}) = 0
$$
where $ \lambda' = \lambda / \eta $. Suppose that $M=(V^T\otimes I_m)   (B^T \otimes I_d) \in R^{m \times dn}$  and $C = X^{\text{val},T} ( X^\text{val} V -Y^\text{val}) \in R^{nd \times 1}$, we have
$$
A M^{-1} 1_m =\lambda'  C
$$
\begin{eqnarray}
( \lambda L_{e} \otimes I_d +  X^TX) M^{-1} 1_m &=&\lambda'  C  \\
(\lambda L_{e} \otimes I_d )M^{-1} 1_m &=& \lambda'  C  - X^TX M^{-1} 1_m \\
\lambda( (E \diag{e} E^T ) \otimes I_d )u  &= &\lambda'  C  -X^TX u \\
( (E \diag{e} E^T ) \otimes I_d )u  &= & v\\
(E \otimes I_d )( \diag{e}  \otimes I_d ) (E^T  \otimes I_d ) u  &= & v\\
(z^T \otimes (E \otimes I_d )) \mvec{\diag{e}  \otimes I_d } &=& v \\
(z^T \otimes (E \otimes I_d )) K e & =& v
\end{eqnarray}
where $u = M^{-1} 1_m$, $v = 1/\eta C  -1/\lambda X^TX u$, $z = (E^T  \otimes I_d ) u  $. The matrix $K = [\mvec {\diag{d_0} \otimes I_d},\dots, \mvec{\diag{d_{m-1}} \otimes I_d}] \in R^{m^2d^2 \times m}$, where $d_i \in R^{m \times 1}$ is the indicator vector.
\end{proof}

\begin{proof}[Theorem \ref{th:complex}]
In $A = \lambda L_{e} \otimes I_d +  X^TX $, the first term is a block diagonal matrix, each block is a Laplacian matrix that has $4m$ non zero elements, in total $O(dm)$. The second element is a block diagonal matrix whose blocks are $X_i^TX_i \in R^{d \times d}$ of $d^2$ entries, in total $nd^2$. In total, computing $A$ has complexity $O(nd^2+dm)=O(nd^2)$. 
Since $A$ for Thm.2 is a SSD, solving in $A$   \cite{spielman_solving_2004,spielman_nearly-linear_2012} has complexity is $O((nd)^{1.31})$ or $O((nd)\ln ^c(nd) )$, with $c$ constant, since $A$ has dimension $nd \times nd$ and the number of edges is $m=kn$. Since the matrix $B \otimes I_d$ has $4md$ non zero elements and $V\otimes I_m$ has $ndm$ non zero elements, the product of the  $(V^T\otimes I_m)(B^T \otimes I_d)c$ requires $O(dnm)=O(dn^2)$ operation, for $c \in R^{dn \times 1}$.  
\end{proof}

\subsection{The $\ell_2^2$-GGMTL and  $\ell_2$-GGMTL algorithms} \label{sec:GGMTL}
This section attempts to clarify the difference of the two variations of the proposed method. The first modification of Alg.\ref{alg:GGMTL} with $\ell^2_2$ norm (Sec.\ref{sec:l22-GGMTL}) is presented in Alg.\ref{alg:l22-GGMTL}, while the modification with $\ell^2_2$ norm (Sec.\ref{sec:l2-GGMTL}) is described in Alg.\ref{alg:l2-GGMTL}.

\begin{algorithm}[!t]
	\SetNoFillComment
	\SetKwInOut{Input}{Input}
	\SetKwInOut{Output}{Output}
	\Input{$\{X_t, y_t\}$ for $t=\{1, 2, ..., n\}$,  $\xi, \eta,\lambda,\nu$}
	\Output{$V= [w_1^T, ..., w_n^T]^T, L_e$}
	\For{$i\gets1$ \KwTo $n$}{
		Solve $w_i$ by Linear Regression on $\{X_i, y_i\}$
	}
	Construct $k$-nearest neighbor graph $G$ on $V$\;
	Construct $E$ the incident matrix of $G$\;
	$\{X^\text{tr}_t, y^\text{tr}_t\},\{X^\text{val}_t, y^\text{val}_t \} \gets \text{split}( \{X_t, y_t\})$ \Comment{validation-training split}\;
	\While{not converge}{
		compute $\dd _{e} f(V_{e^{(t)}},e^{(t)})$ using Eq. (\ref{eq:hg}) \Comment{compute hyper-gradient}\;
		Update $e$: $e^{(t+1)} = [ e^{(t)}+\nu \dd _{e} f(V_{e^{(t)}},e^{(t)})]_{+}$  \Comment{edges' values update}\;
	}
	Solve Eq.\ref{eq:bi_main_in} on $\{X_t, y_t\}$ \Comment{Train on the full dataset}\;
	\Return $V,L_e$\;
	\caption{$\ell_2^2$-GGMTL}
	\label{alg:l22-GGMTL}
\end{algorithm}

\begin{algorithm}
	\SetNoFillComment
	\SetKwInOut{Input}{Input}
	\SetKwInOut{Output}{Output}
	\Input{$\{X_t, y_t\}$ for $t=\{1, 2, ..., n\}$,  $\xi, \eta,\lambda, \nu$}
	\Output{$V= [w_1^T, ..., w_n^T]^T, L_e$}
	\For{$i\gets1$ \KwTo $n$}{
		Solve $w_i$ by Linear Regression on $\{X_i, y_i\}$
	}
	Construct $k$-nearest neighbor graph $G$ on $V$\;
	Construct $E$ the incident matrix of $G$\;
	\tcp{validation-training split}
	$\{X^\text{tr}_t, y^\text{tr}_t\},\{X^\text{val}_t, y^\text{val}_t\} \gets \text{split}( \{X_t, y_t\})$  \; 		
	\While{not converge}{
		\tcp{compute hyper-gradient}
		compute $\dd _{e} f(V_{e^{(t)}},e^{(t)})$ using Eq. (\ref{eq:hg}), with $e=e \circ l(V)$ of Eq.\ref{eq:elle} and alternating with solution of Eq.\ref{eq:bi_ccmtl_sub} (given by Eq.\ref{eq:sol_inner}).\;
		\tcp{edges' values update}		
		Update $e$: $e^{(t+1)} = [ e^{(t)}+\nu \dd _{e} f(V_{e^{(t)}},e^{(t)})]_{+}$\;
	}
	\tcp{Train on the full datasets with alternate optimization}
	Solve Eq.\ref{eq:bi_ccmtl} on $\{X_t, y_t\}$\;
	\Return $V,L_e$\;
	\caption{$\ell_2$-GGMTL}
	\label{alg:l2-GGMTL}
\end{algorithm}

\subsection{Visualization of GGMTL graphs on other datasets} \label{sec:other_datasets}
Two additional datasets are presented and graph visualized to show the property of the proposed method. 

\subsubsection{South America climate data}

\texttt{South} data contains the monthly mean temperature of 250 spatial locations in South America for 100 years (over the time period 1901-2000)\footnote{MSSL code: \url{bitbucket.org/andreric/mssl-code}\\ Weather dataset: \url{https://www-users.cs.umn.edu/~agoncalv/softwares.html}}. These locations are distributed over a $2.5^o $ x $2.5^o$ grid on the geographic coordinate system (latitudes, longitudes), see \cite{goncalves_multi-task_2016}. Following \cite{goncalves_multi-task_2014}, five partitions are created by considering a moving window over the first 50 years for training and a window over the following ten years for testing. By shifting these two window ten years for each partitions, we obtain the five training/testing pairs.
\subsubsection{North America climate data}
Similarly, the \texttt{North} data contains the monthly mean temperature of 490 spatial locations in North America.

\begin{figure}
	\centering
	\begin{subfigure}{.5\textwidth}
		\centering
		\includegraphics[width=1.\linewidth,trim=0 0cm 0 1cm, clip]{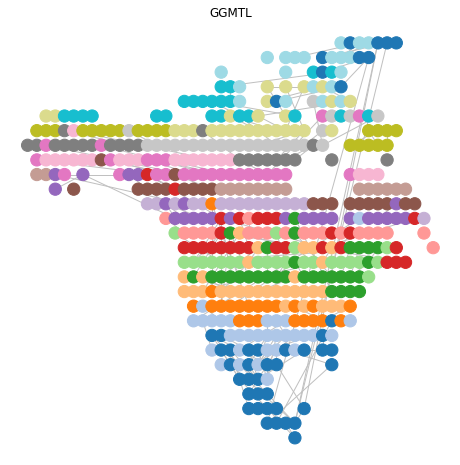}
		\caption{GGMTL}
		\label{fig:GGMTL_north}
	\end{subfigure}%
	\begin{subfigure}{.5\textwidth}
		\centering
		\includegraphics[width=1.\linewidth,trim=0 0cm 0 1cm, clip]{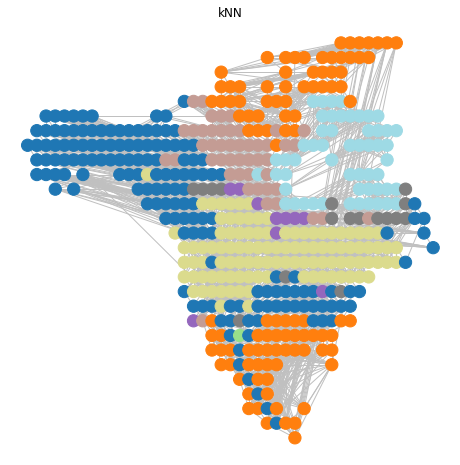}
		\caption{kNN}
		\label{fig:kNN_north}
	\end{subfigure}
	\caption{The discovered graphs by GGMTL and k-NN on the the \texttt{North} datasets.}
	\label{fig:climate_north}
\end{figure}

\begin{figure}
	\begin{subfigure}{.5\textwidth}
		\centering
		\includegraphics[width=1.\linewidth,trim=0 0cm 0 1cm, clip]{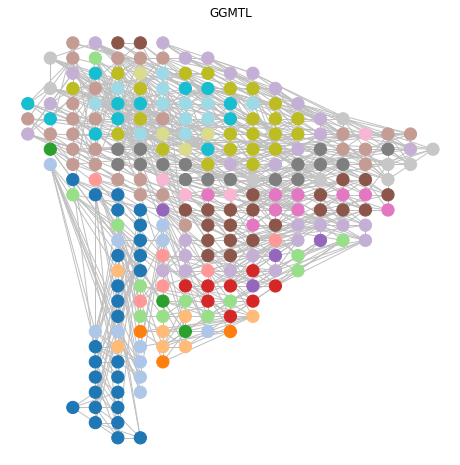}
		\caption{GGMTL (south)}
		\label{fig:GGMTL_south}
	\end{subfigure}%
	\begin{subfigure}{.5\textwidth}
		\centering
		\includegraphics[width=1.\linewidth,trim=0 0cm 0 1cm, clip]{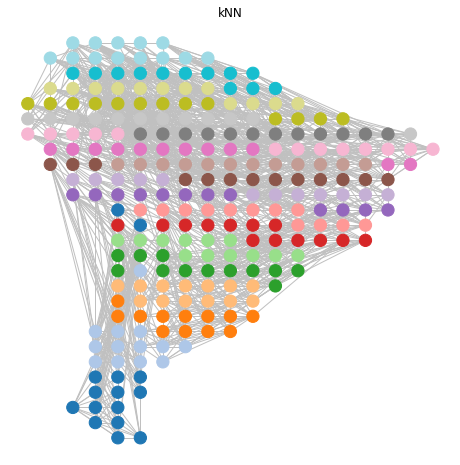}
		\caption{kNN (south)}
		\label{fig:kNN_south}
	\end{subfigure}

	\caption{The discovered graphs by GGMTL and k-NN on the the \texttt{South} datasets.}
	\label{fig:climate_south}
\end{figure}

\end{document}